\documentclass[conference]{IEEEtran}
\usepackage{amsmath,amsfonts,amssymb,amsthm}
\usepackage{times}
\usepackage{comment}
\usepackage{xcolor}
\usepackage{graphicx}
\usepackage{subfigure}
\usepackage{multirow}
\usepackage[numbers]{natbib}
\usepackage{multicol}
\usepackage[bookmarks=true]{hyperref}
\usepackage{makecell} 
\usepackage{dsfont}

\usepackage{floatrow}
\usepackage[linesnumbered,ruled]{algorithm2e}
\usepackage{txfonts}
\newcommand{\bF}{\boldsymbol{F}}
\newcommand{\br}{\boldsymbol{r}}
\newcommand{\bxi}{\boldsymbol{\xi}}
\newcommand{\bA}{\boldsymbol{A}}
\newtheorem{thm}{Theorem}[]

\makeatletter
\newcommand{\printfnsymbol}[1]{%
  \textsuperscript{\@fnsymbol{#1}}%
}
\makeatother

\begin{document}

\title{Agile asymmetric multi-legged locomotion: contact planning via geometric mechanics and spin model duality}

\author{
\IEEEauthorblockN{Jackson Habala$^1$,
Gabriel B. Margolis$^2$,
Tianyu Wang$^3$,
Pratyush Bhatt$^1$,
Juntao He$^3$,
Naheel Naeem$^1$,\\
Zhaochen Xu$^3$,
Pulkit Agrawal$^{2}$,
Daniel I. Goldman$^3$,
Di Luo$^{4,*}$,
Baxi Chong$^{1,*}$
}
\IEEEauthorblockA{$^1$Department of Mechanical Engineering, Pennsylvania State University, University Park, PA 16802, USA\\
$^2$MIT Improbable AI Lab, Maschusetts Institute of Technology, Cambridge, MA 02139, USA\\
$^3$Department of Physics, Georgia Institute of Technology, Atlanta, GA 30332, USA\\
$^4$Department of Physics, Tsinghua University, Beijing 10084, P.R. China\\
}
}

\maketitle

\begingroup
\renewcommand{\thefootnote}{\fnsymbol{footnote}}
\footnotetext[1]{Corresponding authors.}
\endgroup

\begin{abstract} 

Legged robot research is presently focused primarily on bipedal or quadrupedal robots, despite longstanding capabilities to build robots with many more legs to potentially improve locomotion performance. This imbalance is not necessarily due to hardware limitations, but rather to the absence of principled control frameworks that explain when and how additional legs improve locomotion performance. In multi-legged systems, coordinating many (redundant) simultaneous contacts introduces a severe curse of dimensionality that challenges existing modeling and control approaches. As an alternative, multi-legged robots are typically controlled using low-dimensional gaits originally developed for bipeds or quadrupeds. These strategies fail to exploit the new symmetries and control opportunities that emerge in higher-dimensional systems, leading to only marginal performance gains despite substantially increased morphological complexity. In this work, we develop a principled framework for discovering new control structures in multi-legged locomotion. We use geometric mechanics to reduce contact-rich locomotion planning to a graph optimization problem, and propose a spin model duality framework from statistical mechanics to exploit symmetry breaking and guide optimal gait reorganization. Using this approach, we identify an asymmetric locomotion strategy for a hexapod robot that achieves a forward speed of 0.61 body lengths per cycle (a 50\% improvement over conventional gaits). The resulting asymmetry appears at both the control and hardware levels. At the control level, the body orientation oscillates asymmetrically between fast clockwise and slow counterclockwise turning phases for forward locomotion. At the hardware level, two legs on the same side remain unactuated and can be replaced with rigid parts without degrading performance. Numerical simulations and robophysical experiments validate the framework and reveal novel locomotion behaviors that emerge from symmetry reforming in high-dimensional embodied systems.
\end{abstract}

\IEEEpeerreviewmaketitle

\section{Introduction}

\begin{figure}[t]
\centering
\includegraphics[width=1\linewidth]{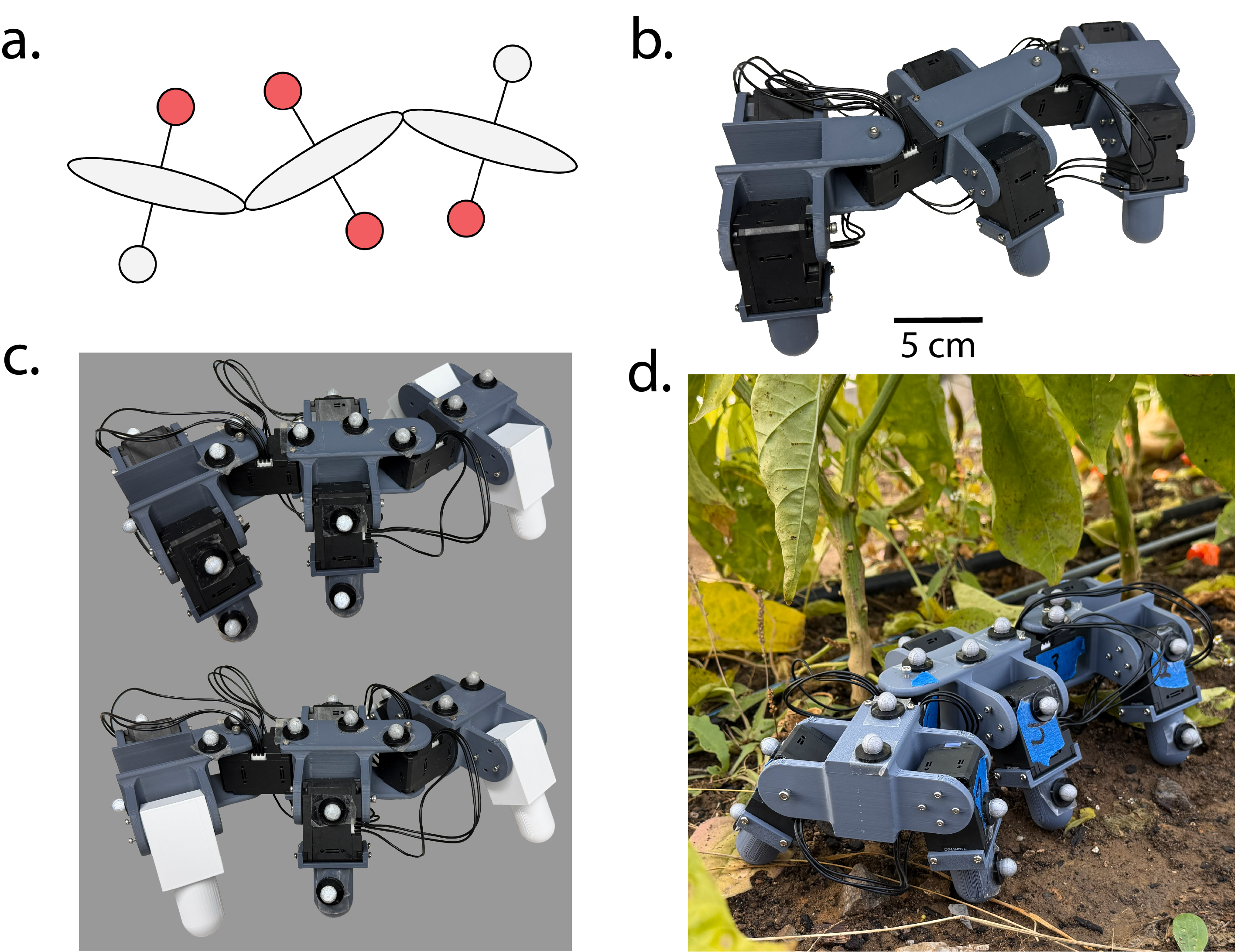}
{\caption{\textbf{Multi-legged robot: a hexapod model.} (a) 
Geometric and (b) robophysical hexapod model. (c) Two physics-informed alternative designs in which two legs are replaced by non-actuated rigid appendages, achieving forward locomotion performance identical to the fully actuated hexapod. (d) Field deployment under a tomato plant.
}
}
\label{fig:introRobot}
\end{figure}

\vspace{-0.5em}

As the cost of motors and sensors continues to decline, building robots with redundant actuation has become increasingly feasible~\cite{verstraten2019kinematically}. However, feasibility alone does not establish necessity. While adding actuators can intuitively improve stability and robustness, similar benefits can often be achieved through advances in perception and control. Indeed, recent progress has demonstrated highly capable bipedal and quadrupedal locomotion using data-driven methods~\cite{margolis2024rapid,choi2023learning,li2025reinforcement} and model-based control algorithms~\cite{grandia2023perceptive,dantec2022whole}. Despite the prevalence of hexapods in biology (more than half of all known species possess 6 legs~\cite{mayhew2007there}),  robotic research in academia and industry remains largely focused on few-legged systems~\cite{lee2020learning}. This disparity raises a fundamental question: Given the demonstrated success of few-legged robots, what functional advantages, if any, do additional legs provide?

We attribute our limited understanding of multi-legged robots to the lack of scalable tools to coordinate and control these systems. Simply adding legs does not guarantee improved locomotion; capturing their benefits requires controllers that effectively coordinate the additional degrees of freedom~\cite{freyberg2023morphological, full1999templates}. Multi-legged locomotion entails coordinating many appendages that make and break contact with the environment, giving rise to contact-rich dynamics~\cite{chong2022self}. The number of possible contact sequences grows rapidly with leg count, leading to an explosive increase in coordination complexity. For example, a hexapod can employ on the order of $10^{89}$ distinct contact sequences (see SI, Sec.~S1), making exhaustive planning or control impractical. Consequently, most existing approaches to hexapod locomotion rely on gait principles adapted from bipedal or quadrupedal robots, which limit our exploration of multi-legged functionality~\cite{Saranli, coelho2021trends}.

Gait principles developed for bipedal and quadrupedal robots typically rely on strict assumptions of symmetry that are essential for few-legged locomotion~\cite{full1999templates, hildebrand1967symmetrical, hildebrand1965symmetrical, Alexander}. The most common is contralateral symmetry, in which left and right legs on the same body segment perform identical motions subject to a phase offset. Under this assumption, whole-body yaw moments generated from the left and right sides are balanced, resulting in straight-line motion. While contralateral symmetry may be necessary for achieving straight-line locomotion in few-legged systems, multi-legged robots, such as hexapods, contain a richer set of symmetries. For example, diagonal and middle-leg reflection symmetries can also generate net forward motion without relying on strict left–right pairing~\cite{shrivastava2020material}. Existing control methods have largely failed to exploit these additional symmetries~\cite{10801714,kim2026deep,he2024learning}, leaving the benefits of deliberately breaking contralateral symmetry unexplored.

In this paper, we develop a physics-informed framework to answer this question: What benefits arise from the new symmetries available in redundant leg systems? We choose a hexapod robot with two body-bending joints as a representative system. To break down the large contact space of the hexapod, we employ geometric mechanics (a framework rooted in differential geometry and geometric phase) to transform the contact-planning problem of multiple legs into a graph optimization problem. Leveraging spatiotemporal symmetries in locomotion, we further map this graph optimization problem to a special class of spin models, which allow us to obtain globally optimal gaits in polynomial time. Using this framework, we identify gait patterns that generate net forward motion through asymmetric left–right leg coordination. The observed gaits produce short-and-rapid clockwise rotation paired with long-and-gradual counterclockwise rotation to achieve straight-line forward locomotion. Experiments show that the resulting gaits achieve speeds up to 0.61 body lengths per cycle (BL/cyc), outperforming conventional approaches, including extended quadrupedal gaits, reinforcement learning–based controllers, and bio-inspired gaits. Additionally, the observed gaits allow for two legs to be replaced with rigid parts without reducing locomotion performance (Fig.~\ref{fig:introRobot}(c)). More broadly, our approach provides a principled tool to analyze multi-legged locomotion without relying on ad hoc dimensionality reduction, and offers a quantitative framework for revealing when and why additional legs fundamentally enhance locomotion performance.

\vspace{-0.5em}
\section{Background}
\vspace{-0.25em}
Consider a regime of locomotion in which inertial forces are negligible compared with frictional forces. We refer to this regime as \emph{geometric locomotion}. Prior work has shown that the relative importance of inertia can be quantified by the \emph{coasting number} $\mathcal{C}$, defined as the ratio between the coasting time (the time required for a locomotor to come to rest after self-deformation ceases) and a characteristic locomotion timescale (typically the gait cycle period)~\cite{rieser2024geometric}. In the low-coasting-number regime ($\mathcal{C} \ll 1$), inertial effects can be safely neglected and locomotion is well described by geometric models. This regime encompasses a broad class of biological and robotic systems, including centipede and snake locomotion, and it remains applicable even at speeds of multiple body lengths per second~\cite{rieser2021functional}.

Low-coasting-number locomotion shares a key property with low-Reynolds-number fluid locomotion~\cite{purcell2014life}: the net displacement over a gait cycle is determined solely by the \emph{sequence} of shape changes, not by the rate at which those shape changes are executed. Consequently, the system dynamics depend only on the evolution of body shapes and not on absolute time scaling. Under this approximation, the equations of motion~\cite{Marsden} reduce to

\vspace{-1.2em}
\begin{equation}\label{eq:EquationOfMotion1}
    \bxi={\bA(\br)\dot \br},
\end{equation}
\vspace{-1.5em}

\noindent where $\bxi=[\xi_x, \xi_y, \xi_\theta]^T$ denotes the body velocity in the forward, lateral, and rotational directions (Fig.~\ref{fig:shapespace}(a)); ${\br}$ denotes the internal shape variables (joint angles, e.g., $\br=[\alpha_1,\ \alpha_2]$ in Fig.~\ref{fig:shapespace}); $\bA(\br)$ is the local connection matrix, which encodes environmental constraints and the conservation of momentum. 

\begin{figure}[t]
\centering
\includegraphics[width=1\linewidth]{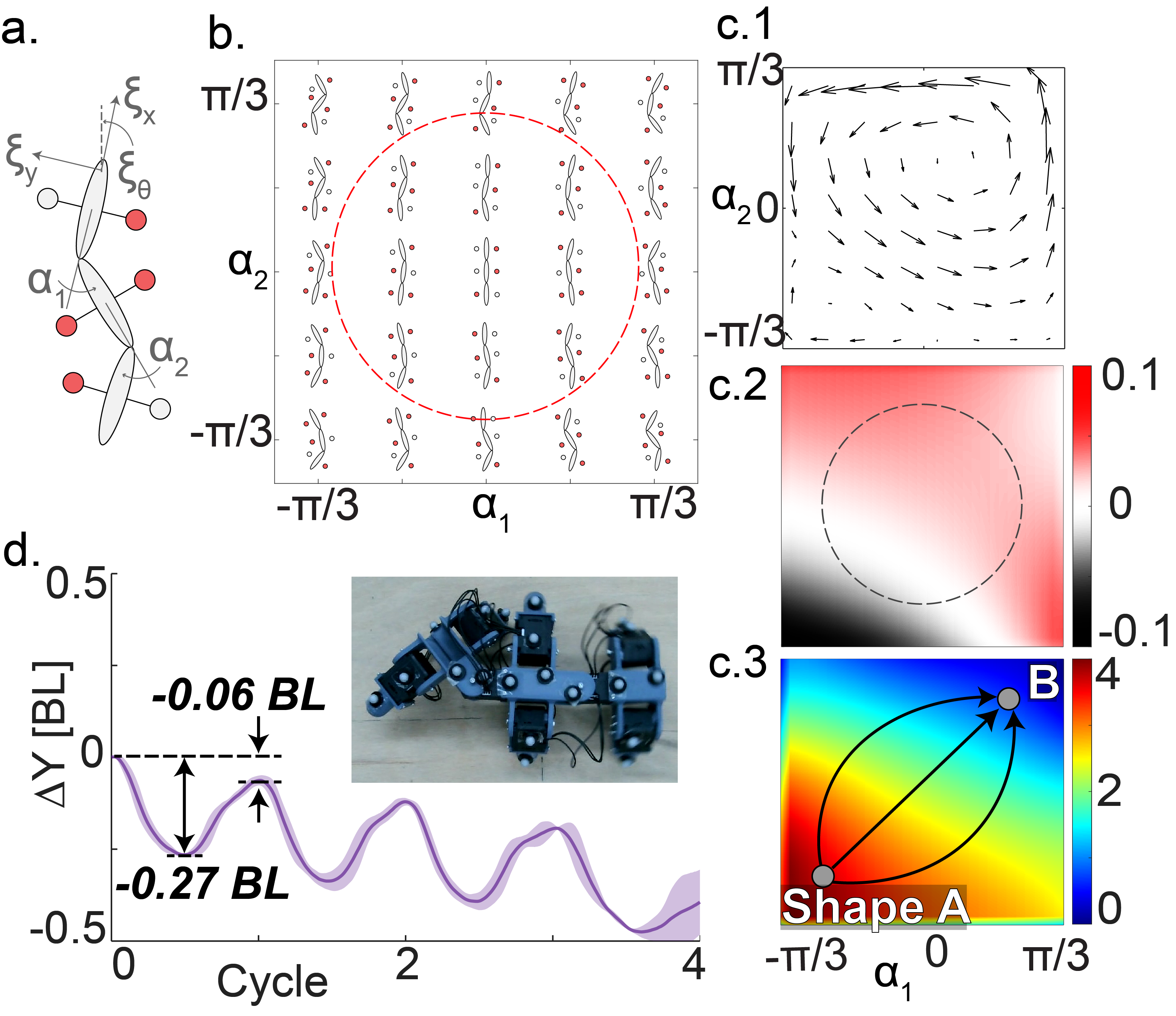}
\caption{\textbf{Geometric mechanics of hexapod locomotion.}
(a) Geometric illustration of hexapod body bending angles $\alpha_1$ and $\alpha_2$ and body velocity ($\xi_x$, $\xi_y$, and $\xi_\theta$)
(b) Shape space of the geometric model parameterized by $\alpha_1,\alpha_2$.
(c.1) Local connection vector field mapping infinitesimal shape velocities to body-frame velocities.
(c.2) Height function over shape space, given by the curl of the divergence-free component of the connection vector field; the net displacement for a closed gait is approximated by the signed area integral of the height function over the enclosed loop.
(c.3) Potential function computed from the curl-free component of the connection vector field. The displacement associated with a shape change (e.g., from shape A to shape B) is gait-path-independent and determined by differences in function values between endpoints.
An order-of-magnitude separation exists between the height function and the potential function over the same shape space.
All panels share identical axes.
(d) Experiment trajectory of a hexapod robot with body undulation evolving without contact changes. In the absence of contact modulation, the robot undergoes oscillatory body motion with peak-to-peak displacement of approximately $0.27$ body lengths (BL), while the net displacement over one cycle remains negligible ($<0.06$ BL).
}
\vspace{-1em}
\label{fig:shapespace}
\end{figure}

The local connection matrix $\bA$ can be numerically derived using resistive force theory (RFT) to model the ground reaction forces (GRF) \cite{li2013terradynamics,sharpe2015locomotor,zhang2014effectiveness}. Specifically, the net GRF experienced by the locomotor is the sum of the GRF experienced by all segments in stance phase (we define a segment to be in the stance phase if it is in contact with the ground, see SI, Sec.~S1). From the geometry and physics of GRF, reaction forces of each foot can be calculated from the body velocity $\bxi$, reduced body shape $\br$, and reduced shape velocity $\dot{\br}$~\cite{rieser2019geometric,murray2017mathematical}.
Assuming quasi-static motion, we consider the total net force applied to the system is zero at any instant in time~\cite{chong2023self}:

\vspace{-0.7em}

\begin{equation}\label{eq:forceIntegral}
    \bF=\sum_{i\in I} {\bF^{i}\left(\bxi,\br,\dot{\br}\right)}=0,
\end{equation}
\vspace{-1em}

\noindent where $I$ is the collection of all stance-phase feet. Per a given body shape $\br$, Eq.~(\ref{eq:forceIntegral}) connects the shape velocity $\dot{\br}$ to the body velocity $\bxi$. Therefore, by the implicit function theorem and the linearization process, we can numerically derive the local connection matrix $\bA(\br)$.
In our implementation, we compute the solution of Eq.~(\ref{eq:forceIntegral}) using the MATLAB function \textit{fsolve}. Importantly, the local connection matrix $\bA(\br)$ is dependent on the contact pattern $I$. 

We then consider the relationship between a gait (a periodic sequence of shape changes, represented as a closed loop in Fig.~\ref{fig:shapespace}(b)) and the resulting body displacement. The displacement along the gait path $ \chi$ can be approximated by~\cite{hatton2015nonconservativity}:
\vspace{-1.0em}
\begin{equation}\label{eq:lineintegral}
    g(T)
    \approx \int_{\chi(0)}^{\chi(T)} {\bA(\br)\mathrm{d}\br} =\int_{\chi(0)}^{\chi(T)}\begin{bmatrix}
\bA^x (\br) \\
\bA^y (\br)\\
\bA^\theta (\br)\\
\end{bmatrix}d\br,
\end{equation}
\vspace{-1em}

\noindent where $\bA^x(\br), \bA^y(\br), \bA^\theta(\br)$ are the three rows of the local connections, corresponding to forward, lateral, and rotational vector fields. The accuracy of the approximation in Eq.~(\ref{eq:lineintegral}) can be optimized by properly choosing the body frame \cite{hatton2015nonconservativity,linoptimizing}. An example vector field ($\bA^x$) is shown in Fig.~\ref{fig:shapespace}(c.1). With the above derivation, the net displacement over a gait period can be approximated with the line integral along the gait path in the local connection vector field. Accordingly, we record the step length in our experiments, defined as the body lengths traveled per cycle (BL/cyc), to match the per-cycle gait path optimization.

\section{Methods}
\subsection{Contact sequence optimization}

In the prior section, we only considered the locomotion problem with a fixed contact pattern, i.e., $I$ is independent of shape variables. Now we consider the locomotion problem where the contact pattern changes. First, we will explore the contact planning problem with a known shape change sequence. Specifically, let $Q=\big\{r_i,\ i\in\{1,\ 2,\ ...\ M\}\ |\ r_i\in\mathbb{R}^2\big\}$ be a collection of sequenced shape variables. We restrict contact switches to occur only in one of the designated shape variables in $Q$. Considering a gait path $\chi$ sequentially connecting all shapes in $Q$, we aim to identify the optimal contact sequence and gait path $\chi$ that maximizes forward displacement.

We first decompose the closed-loop gait path $\chi$ into piece-wise curves $\chi_i$ connecting $r_i$ to $r_{i+1}$. The forward displacement from $\chi$ is then:
\vspace{-1.5em}
\begin{equation}\label{eq:piecewiseint}
    \Delta x= \sum_{i=1}^{M} \int_{ \chi_i} {\bA^x_{I(i)}(\br)\mathrm{d}\br},
\end{equation}

\noindent where $I(i)$ is the contact pattern along the piece-wise curves $\chi_i$, and $\bA^x_{I(i)}$ is the forward vector field evaluated at the contact pattern $I(i)$. Note that the optimization in Eq.~(\ref{eq:piecewiseint}) can be decoupled and separately optimized:

\vspace{-1em}
\begin{align}\label{eq:piecewiseint_sep}
    \max_{\chi,\ I(\chi)} \Delta x = & \max_{\chi,\ I(\chi)} \ \sum_{i=1}^{M} \int_{ \chi_i} {\bA^x_{I(i)}(\br)\mathrm{d}\br} 
\end{align}
\vspace{-0.5em}

We now consider the decoupled optimization problem. Specifically, we seek to optimize $\int_{\chi_i} \mathbf{A}^x_{I(i)}(\mathbf{r}) \cdot \mathrm{d}\mathbf{r}$. We argue that, in most practical applications (particularly in environments governed by dry Coulomb friction), the vector field $\mathbf{A}^x_{I(i)}$ can be well approximated as conservative. Under this approximation, the resulting displacement is independent of the specific path $\chi_i$ taken in the shape space and depends only on the endpoints of the path (Fig.~\ref{fig:shapespace}(c.3), \textit{Shape A to Shape B}). 

To test this assumption, we use the Hodge-Helmholtz decomposition theorem~\cite{guo2005efficient} to decompose the original vector field into an irrotational (curl-free) component and a solenoidal (divergence-free) component: $\mathbf{A}^x_{I(i)} = \mathbf{A}^{x,i}_{I(i)} + \mathbf{A}^{x,s}_{I(i)}$. The magnitude of $\mathbf{A}^{x,i}_{I(i)}$ is characterized by its associated potential function $P$, satisfying $\mathbf{A}^{x,i}_{I(i)} = \nabla P$, such that
$\int_{\chi_i} \mathbf{A}^x_{I(i)}(\mathbf{r}) \cdot \mathrm{d}\mathbf{r}=P(\mathbf{r}_{\mathrm{end}}) - P(\mathbf{r}_{\mathrm{start}})$, which is independent of the path $\chi_i$. In contrast, the magnitude of $\mathbf{A}^{x,s}_{I(i)}$ is characterized by the height function $H = \nabla \times \mathbf{A}^{x,s}_{I(i)}$, for which the displacement generated by a closed-loop path $\chi$ can be estimated by the signed area enclosed by $\chi$ weighted by $H$.

We compute the vector field of the hexapod (lifting front-left and mid-right legs, Fig.~\ref{fig:shapespace}(c.1)), along with the corresponding height function (Fig.~\ref{fig:shapespace}(c.2)) and potential function (Fig.~\ref{fig:shapespace}(c.3)). In both gaits, we observe that the scale of the potential function exceeds that of the height function by more than an order of magnitude, justifying the approximation of the vector field as predominantly conservative. Practically, this implies that although large back-and-forth oscillations in forward displacement may occur during a gait cycle, the cycle-average net displacement remains negligible when the contact pattern is held fixed (illustrated in Fig.~\ref{fig:shapespace}(c.2) by the path along the circle). To validate this hypothesis, we track the forward displacement of a physical hexapod executing a standard body undulation gait with no contact changes. Despite substantial oscillatory motion, reaching amplitudes of up to $0.26$ body lengths (BL), the net displacement over one gait cycle remains negligible (less than $0.06$~BL). 

Under this approximation, let $P^x_{I(i)}$ denote the potential function associated with the curl-free component of $\mathbf{A}^{x,i}_{I(i)}$. Equation~\ref{eq:piecewiseint_sep} can then be simplified to
\vspace{-0.5em}
\begin{equation}\label{eq:piecewiseint_final}
    \max_{\chi,\ I(\chi)} \Delta x  
     = \max_{\chi,\ I(\chi)} \sum_{i=1}^{M}  \Big( P^x_{I(i)}(r_i) -  P^x_{I(i)}(r_{i+1}) \Big).
\end{equation}
\vspace{-0.5em}

Thus, with linear and decoupled terms, Eq.~(\ref{eq:piecewiseint_final}) simplifies the contact planning problem to a linear optimization problem.

\begin{figure}[ht]
\centering
\includegraphics[width=1\linewidth]{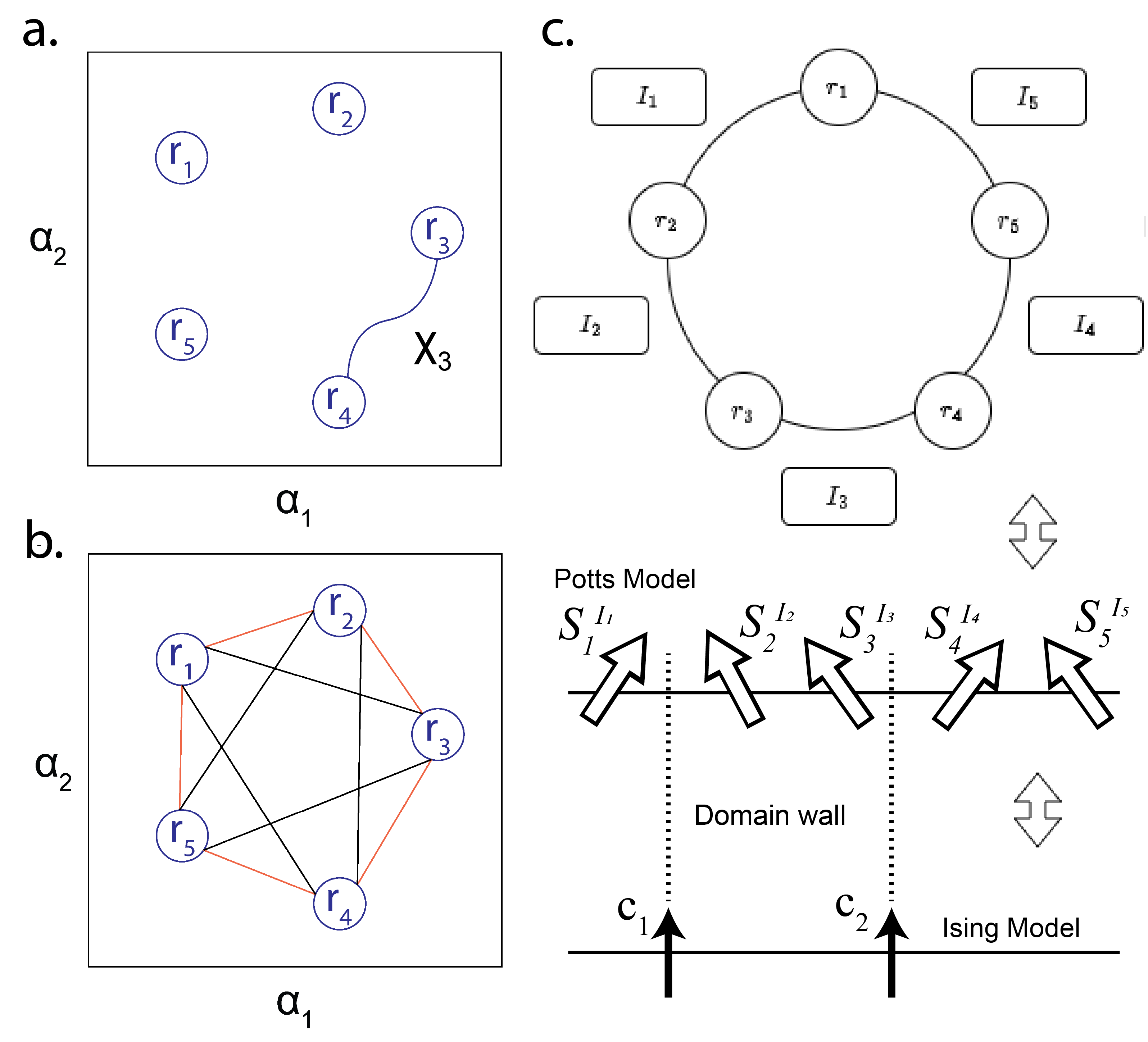}
\caption{\textbf{Gait path mapping.} (a) We sample the shape space with a collection of shape variables $\{r_i\}$. $\chi_3$ is an example path connection of $r_3$ to $r_4$, only a fraction of a closed-loop gait path. (b) An example of a cheating gait path is colored in black. The unique non-cheating gait path is illustrated in red. (c) \textit{(top)} Mapping between robot contact pattern and Potts model. \textit{(bottom)} Duality between Potts model and Ising model.}
\vspace{-1em}
\label{fig:mapping}
\end{figure}

\subsection{Shape change sequence optimization}

In the previous section, we assumed a pre-fixed shape change sequence and formulated a linear optimization problem for contact planning. Here, we explore the shape change sequence optimization. First, we sample shapes in the shape space. Let $Q_{all}=\big\{r_i,\ i\in\{1,\ 2,\ ...\ M\}\ |\ r_i\in\mathbb{R}^2\big\}$ be a collection of sampled shape variables (Fig.~\ref{fig:mapping}(a)). Let $I_{all}$ (with $N$ elements) be a collection of contact states of a multi-legged locomotor. The maximized displacement $\Delta x(C)$ can be approximated as:
\vspace{-0.25em}
\begin{equation}\label{eq:piecewiseint_final_2}
    \Delta x(C) = \max_{C, \ I(C)} \sum_{i\in C} \Big( P^x_{I(i)}(r_{c(i)}) -  P^x_{I(i)}(r_{c(i+1)}) \Big),
\end{equation}

\noindent where $C\subset M$ is the sequence of shape changes, and $\Delta x (C)$ is the resulting displacement. Note that $C$ specifies not only the shape elements, but also the sequence of shape changes. In other words, even with the same elements, there could still be different combinations of sequences in $C$. However, many combinations can be considered as `cheating'. For example, consider a collection of shapes in Fig.~\ref{fig:mapping}(b) with two combinations of sequences colored in red and black. In our framework, we seek to maximize the displacement within \emph{one} period. Therefore, the black path in Fig.~\ref{fig:mapping}(b) has an unfair advantage over the red path, because the black path winds around the origin twice in one period. To avoid `cheating', we define that the shapes can only change in one direction (e.g., clockwise). Thus, given elements in $C$, there is only one valid path connecting all elements. 

In practice, it will always cost finite time and energy to change contact patterns. To account for this cost, we introduce some penalties in contact switching to the cost function. 

\vspace{-1em}
\begin{equation}\label{eq:finalprob}
    \max_{C, \ I(C)}\ \bigg(-\lambda S_{I(C)}+ \sum_{i\in C}  \Big( P^x_{I(i)}(r_{c(i)}) -  P^x_{I(i)}(r_{c(i+1)}) \Big) \bigg),
\end{equation}

\noindent where $S_{I(C)}$ is the number of contact switches over one period, and $\lambda$ is some penalty coefficient. We can further simplify the above optimization problem as a graph optimization problem. Define $V$ as a collection of nodes $v_{ij}$, where the index $j$ denotes the shape and the index $i$ denotes the contact pattern. At each time step, we can either change the contact pattern or shape. Thus, an edge exists connecting two vertices $v_{ij}$ and $v_{kl}$ (if $i=k$ and $j=l$). Let $D$ be a set of weights on the edges. For an edge connecting $v_{ij}$ and $v_{il}$, its weight $d_{ijil} = P^{I(i)}(r_j) -  P^{I(i)}(r_{l}) \in D$. Note that the weight $d_{ijil}$ can be negative. For an edge connecting $v_{ij}$ and $v_{kj}$, its weight $d_{ijkj}$ is 0, i.e., changing the contact pattern at a fixed shape will not cause displacement. With the above notation, the optimization problem in Eq.~(\ref{eq:finalprob}) can be reformulated to find a cycle $(V,D)$, such that the sum of weights along all edges together with $-\lambda S_{I(C)}$ is maximal. 

Owing to the path-independence of the connection vector field, the displacement term $d_{ijil}$ satisfies the following properties:
(1) \emph{Antisymmetry}: $d_{ijil} = - d_{ilij}$; and
(2) \emph{Additivity}: $d_{ijil} + d_{ilik} = d_{ijik}$.
We exploit these properties to reformulate the optimization problem using tools from statistical mechanics.

\subsection{Spin model mappings and solutions}

In general, the coupled contact planning and the shape change optimization problem requires a search in a space of size $N^M$, where $N$ is the number of contact patterns and $M$ is the number of shape variables. A brute-force search will scale poorly as the number of legs and joints increases. To obtain a solution efficiently, we develop a connection between the gait optimization problem and two well-known spin models from statistical physics, the Potts model~\cite{potts1952some} and the Ising model~\cite{mckeehan1925contribution}. In this framework, the optimal shape sequence is equivalent to the ground state of the spin models. In the following sections, we describe the mappings to the Potts model and the Ising model and how it exploits the special structure of the problem to enable efficient gait optimization.

We start with mapping the shape sequence problem to a $\mathrm{Z}_N$-spin model. A $\mathrm{Z}_N$-spin model considers an array of $M$ spins sitting on $M$ sites, where each spin has $N$ internal states. Each spin can interact with other spins or experience an onsite magnetic field. This model is well-known to describe magnetism in physics, where each spin can be thought of as a small magnet. In particular, for $N=2$, it is reduced to the famous Ising model~\cite{mckeehan1925contribution}. It is known that Ising models have deep connections to graph optimization problems and provide equivalent formulations of many NP-complete and NP-hard problems~\cite{lucas2014ising}.

\subsection{Potts model}

If we treat the shape variable $r_j$ as the site index $j$ and the contact pattern $I_j$ as the spin internal state, the pair $(r_j,I_j)$ can be mapped to a spin variable $s_j^i$ (or equivalently $s_j^{I_i}$ in Fig.~\ref{fig:mapping}(c)) at site $j$. The displacement $P^x_{I(i)}(r_{c(j)}) -  P^x_{I(i)}(r_{c(j+1)})$ can be associated with the onsite magnetic field energy $d'_{ij(j+1)} s^i_j$ of the spin variable by defining $s^i_j =e^{2\pi i/N}$ and $d'_{ij(j+1)}=d_{ij(j+1)}e^{-2\pi i/N}$. Here $d'_{ijl}$ is interpreted as the onsite magnetic field for the $\mathrm{Z}_N$-spin model. Since the shape sequence optimization problem is defined on a loop, the problem can be mapped to a $\mathrm{Z}_N$-spin model with periodic boundary conditions as follows:

\vspace{-1em}
\begin{equation}
    H = -\sum_j d'_{ij(j+1)} s^i_j
\end{equation}
\vspace{-1em}

The optimal shape sequence that maximizes Eq.~(\ref{eq:piecewiseint_final_2}) is equivalent to the ground state of the Hamiltonian, i.e., the configuration that minimizes the energy. In this case, the ground state is achieved by each individual spin: $s^i_j$  follows the largest $d_{ij(j+1)}$. This is also known as the greedy solution. Translating back to the shape sequence language, the greedy algorithm chooses the $I(i)$ that maximizes $P^x_{I(i)}(r_{c(i)}) -  P^x_{I(i)}(r_{c(i+1)})$ for each $r_{c(i)}$. 

%

If we further consider the cost of changing contact pattern with penalty $\lambda$ in Eq.~(\ref{eq:finalprob}), it is equivalent to introducing a nearest spin interaction with coupling $\lambda$, which gives rise to a Potts model~\cite{potts1952some} with a nonuniform magnetic field 
\vspace{-0.5em}
\begin{equation}
    H_s = -\sum_j d'_{ij(j+1)} s^i_j - \lambda \delta_{s^i_j, s^i_{j+1}} \label{eq:potts}
    \vspace{-1em}
\end{equation}

\noindent where $\delta_{s^i_j, s^i_{j+1}}=1$ if $s^i_j=s^i_{j+1}$ and 0 otherwise.

The mapping between the shape sequence and the Potts model is also described in the upper part of Fig.~\ref{fig:mapping}(c). For this model, the search complexity is $N^M$. The ground state is expected to have different phenomena depending on $\lambda$. For small $\lambda$, it is disordered, where each spin follows the local field, and the optimal shape sequence is close to the greedy solution. For large $\lambda$, the spins align, and it is ferromagnetic. The optimal shape sequence approaches a uniform gait pattern or a gait pattern with only a small number of jumps.

\subsection{Ground state duality to Ising model}
\label{Sec:Ising}
For the Potts model in Eq.~(\ref{eq:potts}), we further define a domain wall $c_{j+1}$ to be the existence of a different spin configuration between $s_j$ and $s_{j+1}$, where $c_{j+1}=1$ if the domain wall exists and otherwise 0. The mapping between the Potts model and the Ising model is shown in the bottom part of Fig.~\ref{fig:mapping}(c). It follows that we can establish the ground state duality between the Potts model and a long-range Ising model using the domain wall mapping and the additive property of $d_{ijl}$. 

\begin{thm}
The Potts model in Eq.(\ref{eq:potts}) with $d_{ijl}$ satisfying the additive property has the same ground state energy as the following Ising model with long-range couplings $J_{jl}$.

\vspace{-1em}
\begin{equation}
    H_c = -\sum_j J_{jl} c_j c_l - \lambda \sum_i c_i 
    \label{eq:domain}
    \vspace{-1.5em}
\end{equation}
where $J_{jl} = \textup{max}_i d_{ijl}$.
\end{thm}
\textit{Proof.} See SI, Sec.~S2

In general, the Potts model would not be dual to the Ising model due to the fact that they have different spin symmetries. However, duality exists because we only study the ground state and additive property of $d_{ijl}$, which is special in the robotic contact problem. The mapping reduces the complexity of the original problem from $N^M$ to $2^M$. Based on the Ising model Hamiltonian, we develop a domain-wall search algorithm.

\textbf{Domain-wall search algorithm.} Specify the number of domain-walls or jumps $K$ to be searched. Compute the energy $-\sum J_{ij} c_j c_i$ of all $M$-site spin configurations that have $K$ numbers of $c_j=1$. Choose the lowest energy configuration among all as the solution and translate it back to the shape sequence according to the mapping.

There are several advantages of the duality and the domain-wall search algorithm. First, it is independent of the number of shape variables $N$ so that one can always work with a lower degree representation $\mathrm{Z}_2$-spin instead of a $\mathrm{Z}_N$-spin. This is particularly helpful for generalization to a multi-legged robot with many shape variables. Second, the domain wall has a natural interpretation as the jump in the shape sequence problem. In practice, if we would like to have a small number of jumps $K$, it is equivalent to imposing total spin symmetry number or magnetization $K$ to the model. Without such symmetry, the original model in Eq.~(\ref{eq:domain}) lives in a Hilbert space of size $2^M$ where $M$ is the number of shape variables. The symmetry constraint helps to reduce the Hilbert space complexity to $\text{Cr}(M,K)$. We can also ignore the penalty term $\lambda$ in this dual formulation within a fixed spin symmetric sector, because the penalty will only contribute a constant shift to the energy in the same sector. Therefore, the domain-wall search algorithm can reach the optimal solution of the problem in polynomial time when $K$ is finite. This approach provides the foundation for designing optimal locomotion with physics-informed gaits.


\vspace{-0.3 em}
\section{Experimental Setup}
\vspace{-0.3 em}

To test our method, we employ a hexapod robot (which we will refer to as the robophysical model) featuring two rotational degrees of freedom in the body and one rotational degree of freedom in each leg. The elongated body, measuring $26.5$ cm in length, is segmented by bending joints that allow for rotations of up to $\pm 5\pi/16$ without causing self-collision. Each segment is equipped with limbs on both sides, which can be lifted and landed through control of the rotational joint at the shoulder. All rotational joints are actuated by Dynamixel XL430-W250-T servo motors, and the body linkages and appendages are 3D-printed from PLA. Experiments were conducted on a flat, hard surface and were based on the assumption of dry Coulomb kinetic friction. Servo motors were directed using position commands executed by onboard PD control. All gaits were tested with three repeated trials. Each trial consisted of five complete gait cycles. Unless otherwise noted, values are reported as mean~$\pm$ standard deviation (SD) across the three repeated trials. Robot motion was captured using a Vicon motion capture system, comprised of six Vero optical motion capture cameras. 

\begin{figure}[t!]
\centering
\includegraphics[width=1\linewidth]{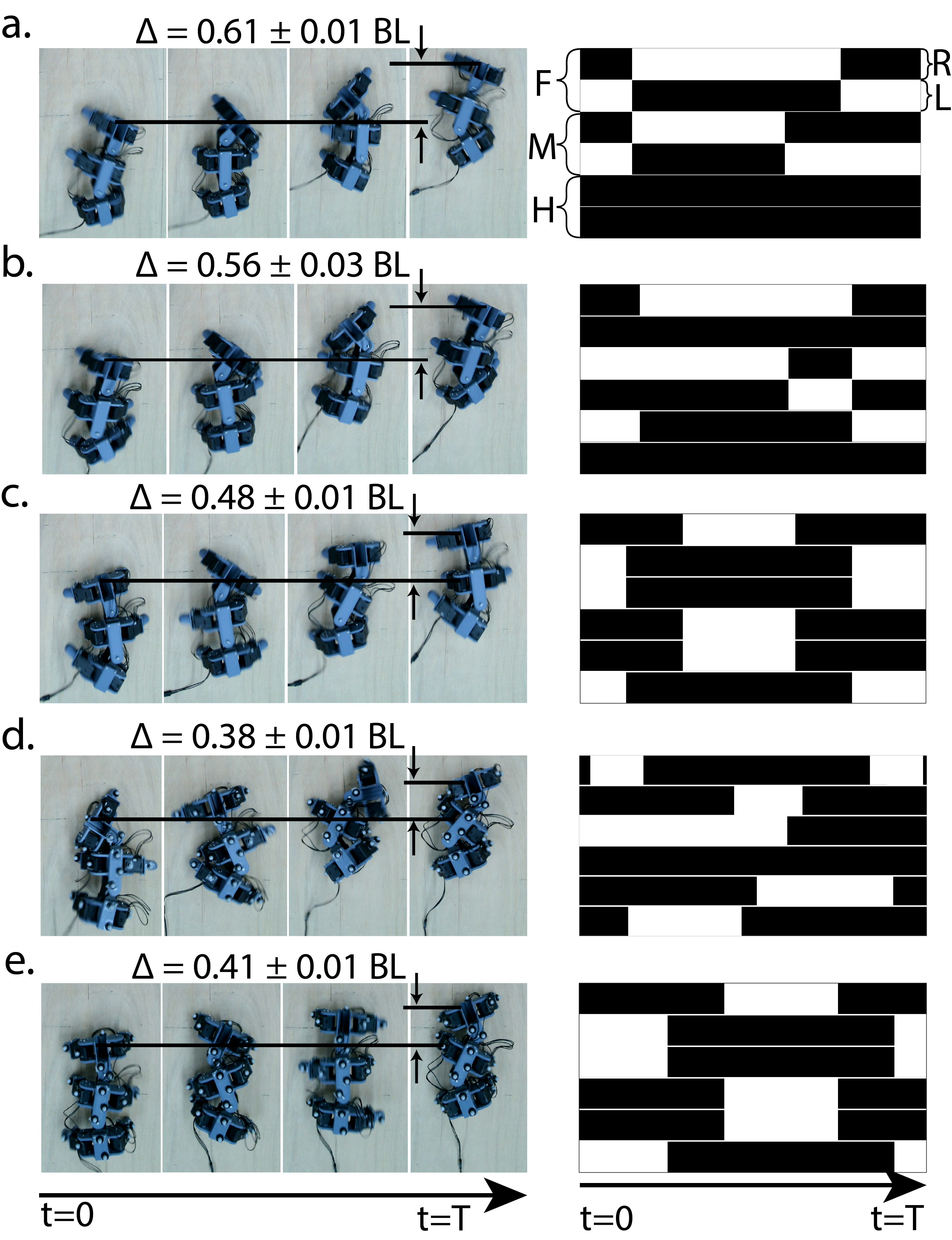}
\caption{\textbf{Comparison of hexapod forward locomotion.}
(\textit{Left}) Time-lapse snapshots of the robot executing the prescribed gaits.
(\textit{Right}) Corresponding commanded leg–ground contact patterns.
(a) Physics-informed asymmetric gait exhibiting slow, long-duration counterclockwise whole-body yaw (from $t=0$ to $\approx0.66T$) followed by rapid, short-duration clockwise yaw, resulting in net forward translation. Hind legs remain in contact throughout gait cycle.
(b) Alternative physics-informed gait in which the front-left and hind-left legs remain in continuous contact throughout the cycle, yet still achieve high forward performance. Similar long-duration counterclockwise and short-duration clockwise whole-body yaw to (a).
(c) Best bio-inspired gait based on an alternating tripod contact pattern.
(d) Best learned open-loop gait obtained via reinforcement learning.
(e) Best extended quadrupedal gait, in which the third body segment follows the same kinematics as the first two segments (effectively reducing the system to a quadruped). Corresponding to $\phi_{\alpha}=\pi$.
}
\vspace{-0.8em}
\label{fig:hexforward}
\end{figure}

\vspace{-0.3 em}
\section{Results}
\vspace{-0.3 em}

We demonstrate the effectiveness of gaits generated by our optimization framework through both simulation and robophysical experiments (Sec.~\ref{Sec:PhysicsGaitResults}), with particularly strong performance at high speeds (Sec.~\ref{Sec:FrequencyResults}). The resulting gaits outperform reference baselines, including bio-inspired tripod gaits (Sec.~\ref{Sec:TripodGaitResults}), extended quadrupedal gaits (Sec.~\ref{Sec:TripodGaitResults}), and a PPO-trained reinforcement learning policy (Sec.~\ref{Sec:RLGaitResults}).

\subsection{Physics-informed Gait}\label{Sec:PhysicsGaitResults}

We use geometric mechanics and the Ising model domain-wall search algorithm to identify new symmetries in hexapod robots for forward locomotion. With six legs, the hexapod has 64 different contact states; however, over half of these states are unstable and cannot be achieved (e.g., no legs in stance phase). For simplicity, we restrict the search to stable contact patterns with at least four legs in stance (22 total contact patterns). We additionally include two three-leg patterns corresponding to the alternating tripod contact states. Thus, in total, we consider 24 contact states as the stable contact space of the hexapod. We sample the shape space by a circular prescription: $\alpha_1 = A_\alpha \sin{(\tau)}, \alpha_2 =  A_\alpha \cos{(\tau)}$, and $\tau$ is uniformly sampled over a period.

We observe that the identified optimal gaits (following our methods in Sec.~\ref{Sec:Ising})  are sensitive to our choice of amplitude $ A_\alpha$. For a range of $ A_\alpha$, the optimization converges to two gaits, corresponding to the optimization for $A_\alpha<7\pi/18$ rad and $A_\alpha>7\pi/18$ rad. We will refer to them as the physics-informed low-amplitude gait and the physics-informed high-amplitude gait. Note that the actual amplitude of the sinusoidal motion is set to $5\pi/16$ due to hardware limitations. However, because the hardware-limit amplitude is near the boundary for physics-informed low-amplitude and high-amplitude gaits, we test both gaits. Snapshots of our robot implementing these gaits are illustrated in Fig.~\ref{fig:hexforward}(\textit{left}). We also illustrate the contact map in Fig.~\ref{fig:hexforward}(\textit{right}). 

We observe that in both gaits, two legs on the hexapod remain in full contact throughout a gait period. In the short-amplitude gait, the hind limbs remain unchanged, and in the large-amplitude gait, the front left and hind left limbs remain in full contact (Fig.~\ref{fig:hexforward}(a-b)). We tested our identified gaits on a robophysical hexapod model and recorded speeds greater than bio-inspired gaits. The small-amplitude (with both hind limbs unactuated) produces a step length of $0.61\pm 0.01$ BL/cyc, whereas the large-amplitude gait (FL and HL limbs unactuated) produces a step length of $0.56 \pm 0.03$ BL/cyc. 

Notably, because two of the legs are not actuated throughout the locomotion, they can be replaced by non-actuated rigid parts. To test this, we construct non-actuated limbs and observe no change in forward displacement (Fig.~\ref{fig:introRobot}(c) and SI movie).

\begin{figure}[t]
\centering
\includegraphics[width=\linewidth]{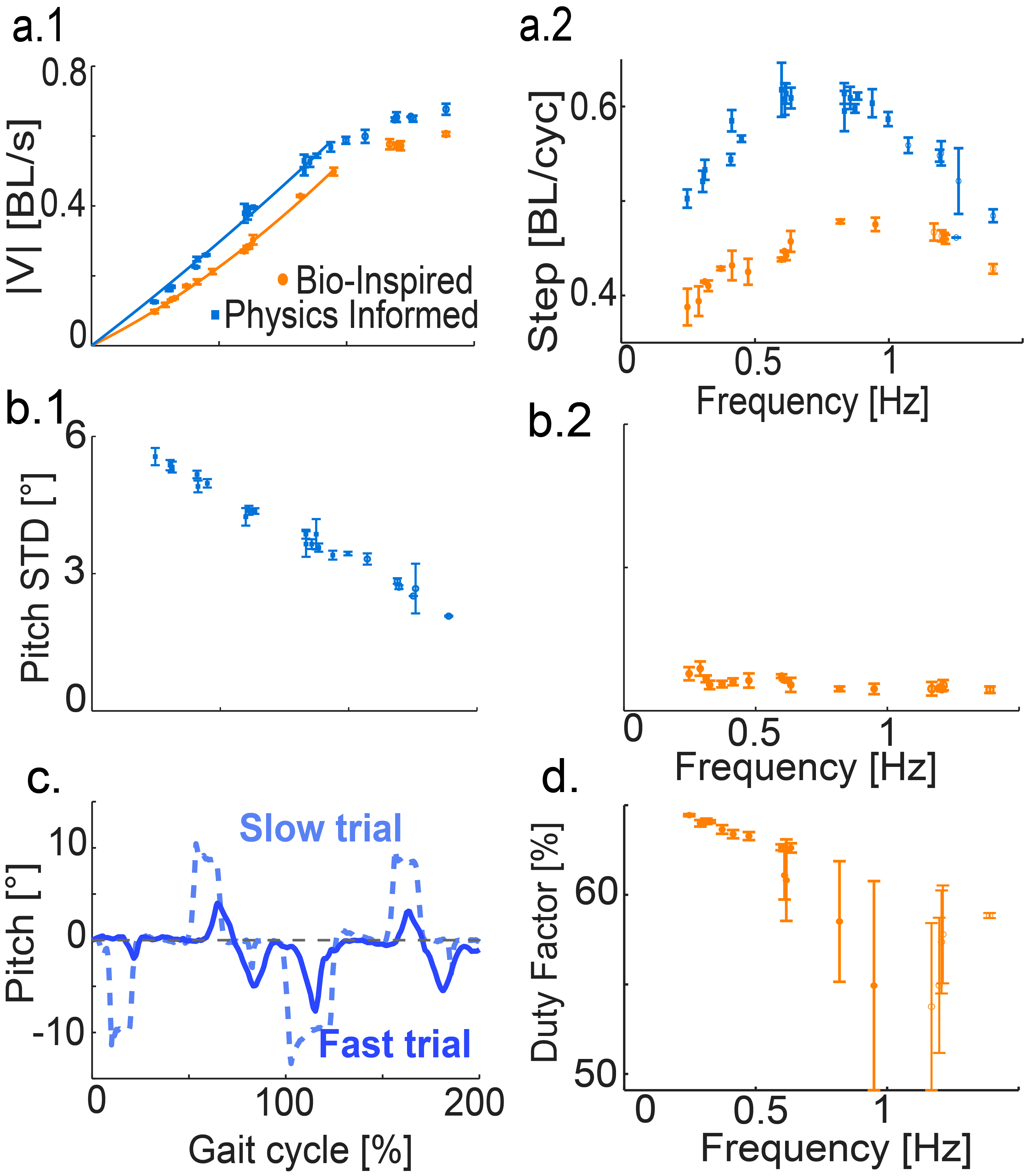}
\caption{\textbf{Agile locomotion at high speed.}
(a) (a.1) Absolute speed and (a.2) step length (displacement per gait cycle) for the bio-inspired and physics-informed gaits as a function of gait frequency. Data above 1~Hz are excluded from quantitative analysis due to motor limitations that cause substantial deviation between commanded and executed position (see SI, Sec.~S3); shown as hollow markers. We excluded outliers from each data set with studentized residuals $> 3$. Because step length increases approximately linearly with frequency, the absolute speed scales quadratically with frequency for both gaits. Panels share identical x-axes.
(b.1) Mean standard deviation of pitch angle as a function of frequency for the physics-informed gait and (b.2) for the bio-inspired gait. Panels share identical axes.
(c) Experimental measurements of whole-body pitch for the physics-informed gait at low and high frequencies over two cycles. At low frequencies, the robot exhibits large pitch oscillations due to loss of static stability near unstable configurations. At high frequencies, dynamically acquired stability compensates for these perturbations, resulting in substantially reduced pitch oscillations.
(d) Illustration of the effect of gait frequency on duty factor: as frequency increases, duty factor decreases, which may partially account for the larger step length observed for the bio-inspired gait at higher frequencies.
}
\label{fig:hexfrequency}
\end{figure}

\subsection{Impact of Frequency}\label{Sec:FrequencyResults}
During gait transitions in the physics-informed gaits, we observe that there is a substantial amount of whole-body pitch during locomotion (see SI movie). We quantify this by measuring the whole-body pitch angle, which is defined as the difference between the initial body vector and current body vector projected onto the forward and vertical axes (a positive pitch corresponds to the body vector pointing up). We posit that this unstable pitch oscillation can be avoided when the robot is operating at a higher frequency, and thus the physics-informed gait should achieve better performance at high speed.

To test this hypothesis, we record the speed dependence of the physics-informed small-amplitude gait (Fig.~\ref{fig:hexfrequency}(a)). Experimental results reveal robot step length (BL/cyc) increases linearly with gait frequency. Thus, there is a quadratic relationship between absolute speed and frequency ($|v_{physics}| = 0.559f + 0.068f^2, R^2=0.9920, p < 0.05$). Notably, data above 1~Hz are excluded from quantitative analysis due to the disparity between commanded and actual position caused by motor limitations (see SI, Sec.~S3). 

Next, we analyze the effect of frequency on body pitch to understand the increasing speed relationship for our physics-informed gait. Comparing two trials at high and low gait frequency, total-body pitch exhibits higher oscillations with lower frequency (Fig.~\ref{fig:hexfrequency}(c)). This relationship is attributed to reduced static stability near unstable configurations. In contrast, higher gait frequencies introduce dynamic stability that reduces perturbations from unstable configurations (i.e., increasing frequency reduces the time spent falling in an unstable configuration). Analysis of the average mean standard deviation at increasing gait frequencies supports this idea Fig.~\ref{fig:hexfrequency}(b.1-2).  

\vspace{-0.2em}
\subsection{Baseline: bio-inspired gait}\label{Sec:TripodGaitResults}
\vspace{-0.2em}

Inspired by insect locomotion, we prescribe the leg–ground contact pattern using an alternating tripod gait, in which the front-right (FR), middle-left (ML), and hind-right (HR) legs form one tripod, while the front-left (FL), middle-right (MR), and hind-left (HL) legs form the other. Each tripod configuration is statically stable. To avoid unstable configurations during contact switching, we introduce an intermediate all-contact phase in which all six legs remain in contact with the ground. As a result, the overall duty factor (defined as the average fraction of a gait cycle during which a leg remains in contact with the ground) is $66.6\%$. Consistent with prior studies, duty factor plays a critical role in determining the locomotion performance of alternating-tripod gaits~\cite{hildebrand1989quadrupedal, chong2022general}. The prescribed contact function is illustrated in Fig.~\ref{fig:tripod}(b).

To simplify our analysis, we prescribe sinusoidal functions to upper and lower body links,  $\alpha_1 = A_\alpha \sin(t)$ and $\alpha_2 = A_\alpha \sin(t+\phi_\alpha)$, respectively (Fig.~\ref{fig:tripod}(a)). $A_\alpha$ is the joint angle limit on the robot, and $\phi_\alpha$ is the phase shift between the upper and lower body joints ($\alpha_1 $ and $\alpha_2$) (Fig.~\ref{fig:tripod}(b)). Additionally, we define $\phi_\beta$ as the phase difference between the upper body joint and the contact sequence (Fig.~\ref{fig:tripod}(b)). Thus, we form a parameter space spanned by $\phi_\alpha$ and $\phi_\beta$. 

We test how these two parameters affect hexapod locomotion through robophysical experimentation and RFT-based model predictions. The RFT-based model agrees well with experimental results subject to a phase shift of $\pi/8$ on $\phi_\beta$ (Fig.~\ref{fig:tripod}(c)). We attribute this phase shift to the finite time it takes for contact switching in robophysical experiments. We observe that the parameter set of $\phi_\alpha = 11\pi/8$ and $\phi_\beta = 0$ produces the highest step length of $0.48 \pm 0.01$ BL/cyc, which we will refer to as the best bio-inspired gait. 

Additionally, when $\phi_\alpha = \pi$, the hexapod gait is effectively reduced to a quadruped gait (the conventional gait) because the additional pair of legs only contributes to stability. We refer to this gait as the extended quadrupedal gait (step length of $0.41\pm0.01$ BL/cyc).

Finally, we record the speed dependence of the best bio-inspired gait (Fig.~\ref{fig:hexfrequency}(a)). Similarly to the physics-informed gait, the robot step length increases linearly with gait frequency. The bio-inspired gait has a quadratic fit of $|v_{bio}|=0.362f + 0.177f^2, R^2=0.9969, p < 0.05$. To explain the quadratic relationship of speed and frequency in the bio-inspired gait, we analyze the actual duty factor of the bio-inspired gait illustrated in Fig.~\ref{fig:hexfrequency}(d). Bio-inspired gaits are assigned a 66.6\% duty factor. As frequency increases, duty factor decreases because of the finite time it takes to switch contact, which we  attribute as the mechanism for increased speed of the bio-inspired gait. We observe increasing gait frequency has no impact on the bio-inspired pitch standard deviation due to the stable static gait configurations and imposed duty factor.

\begin{figure}[t]
\centering
\includegraphics[width=1\linewidth]{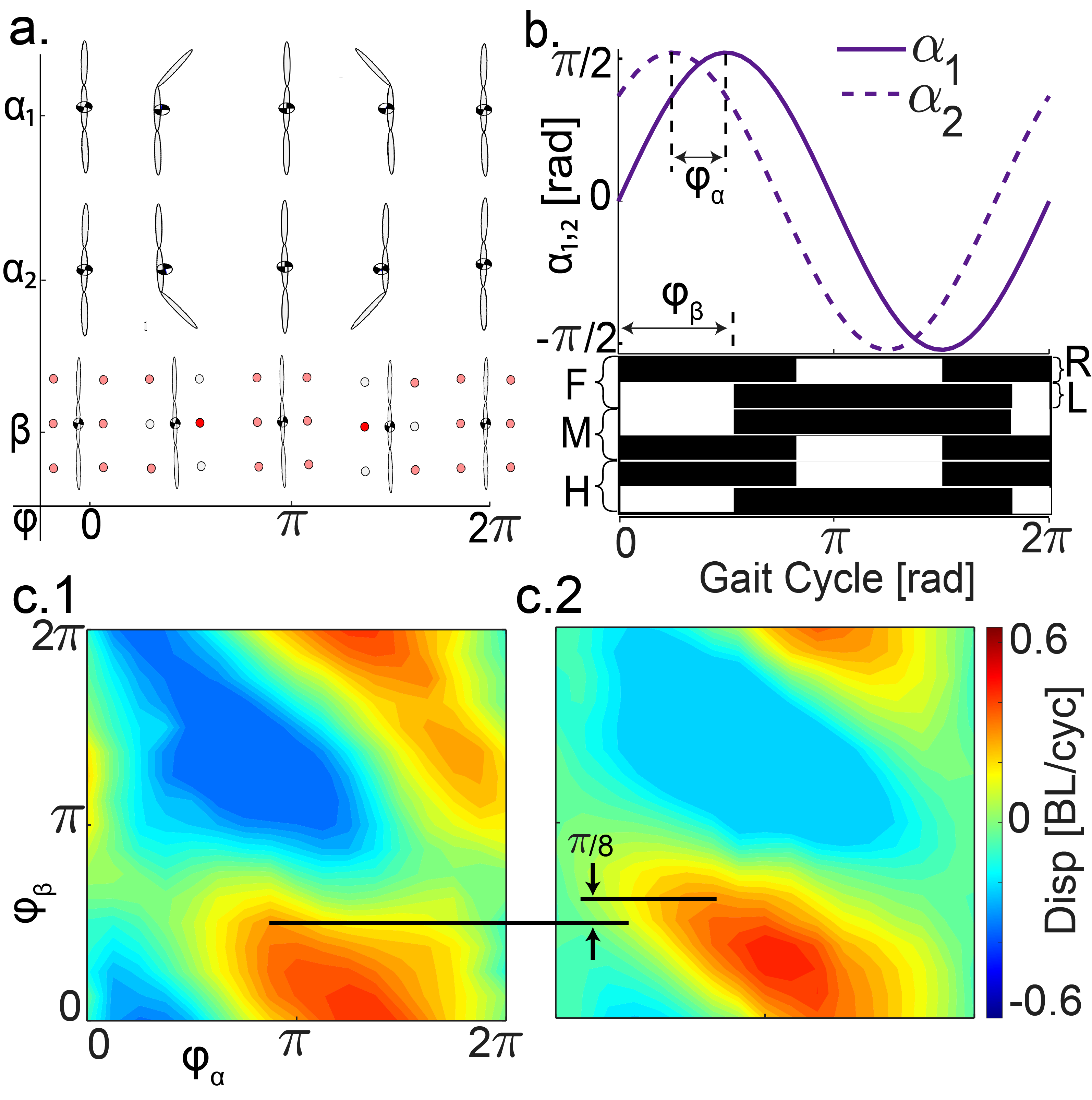}
\caption{\textbf{Bio-inspired hexapod gait.}
(a) Sinusoidal prescriptions of (\textit{top}) upper body joint $\alpha_1$, (\textit{middle}) lower body joint $\alpha_2$, and (\textit{bottom}) alternating tripod contact pattern $\beta$. Legs in stance and aerial phases are indicated by red and open circles, respectively.
(b) Definition of gait parameters: $\phi_\alpha$, the phase shift between $\alpha_1$ and $\alpha_2$; and $\phi_\beta$, the phase shift between $\alpha_1$ and the contact pattern $\beta$. Legs are labeled by position: F (front), M (middle), H (hind), R (right), and L (left).
(c.1) Experimental measurements and (c.2) RFT-based model predictions over the gait parameter space spanned by $(\phi_\alpha, \phi_\beta)$. Experimental results exhibit an effective contact phase shift of $\pi/8$ relative to simulation, attributed to finite leg touchdown time following contact commands.
}
\vspace{-1em}
\label{fig:tripod}
\end{figure}

\subsection{Comparison to Reinforcement Learning}\label{Sec:RLGaitResults}
For many applications requiring quadrupedal or bipedal locomotion, reinforcement learning has become a promising approach to design robust, agile gaits for these systems~\cite{margolis2024rapid, choi2023learning, li2025reinforcement}; however, these approaches struggle to uncover new or improved gaits for higher-dimensional systems, partially because we lack a principled understanding of multi-legged robot locomotion. To test this, we perform deep reinforcement learning in Isaac Lab on our hexapod model to produce an open-loop gait for forward locomotion. We train a PPO policy for forward-velocity tracking. The action space is continuous joint position commands: at time $t$, the policy outputs $a_t \in \mathbb{R}^{8}$, which specifies desired positions for the 8 actuated joints. The policy is conditioned on a continuous, singular forward-velocity command $v^{\mathrm{cmd}}_t \in \mathbb{R}$ (horizontal and turning commands are set to $0$), and the reward penalizes tracking error between the measured forward velocity $v^x_t$ and $v^{\mathrm{cmd}}_t$. The observation is restricted to an action history of length $o_t = \big[a_{t-1};\, a_{t-2};\, \ldots;\, a_{t-14}\big] \in \mathbb{R}^{112}$, to produce an open-loop gait pattern. The resulting reward curve is illustrated in Fig.~\ref{fig:RL}(a). We used standard reward functions from prior work on learning locomotion~\cite{rudin2022learning} (see SI, Sec.~S4 for learning parameters).

To reduce the sim-to-real gap, we identify the motor dynamics of the robophysical model qualitatively through comparison of motor response curves (see SI, Sec.~S4). Additionally, we analyze the displacement of the learned gait through simulation in Isaac Sim and on our robophysical model (Fig.~\ref{fig:RL}(b)). The displacement was fitted to 5 cycles of the learned gait in the simulation, because the deployed gait frequency on the robophysical model was slower than that of simulation. The displacement results show a low sim-to-real gap of our learned gait. Our open-loop reinforcement learning approach achieves an experimentally tested step length of $0.38$ BL/cyc, underperforming both bio-inspired and physics-informed gaits (Fig.~\ref{fig:hexforward}(d)). The learned gait produces a gait pattern similar to that of the extended quadrupedal gait, in which the front and hind sets of limbs were used in a quadrupedal trotting fashion while the middle pair provided stability (see SI movie). Our analysis reveals an increased difficulty for reinforcement learning to take advantage of all pairs of legs.

\begin{figure}[t]
\centering
\includegraphics[width=1\linewidth]{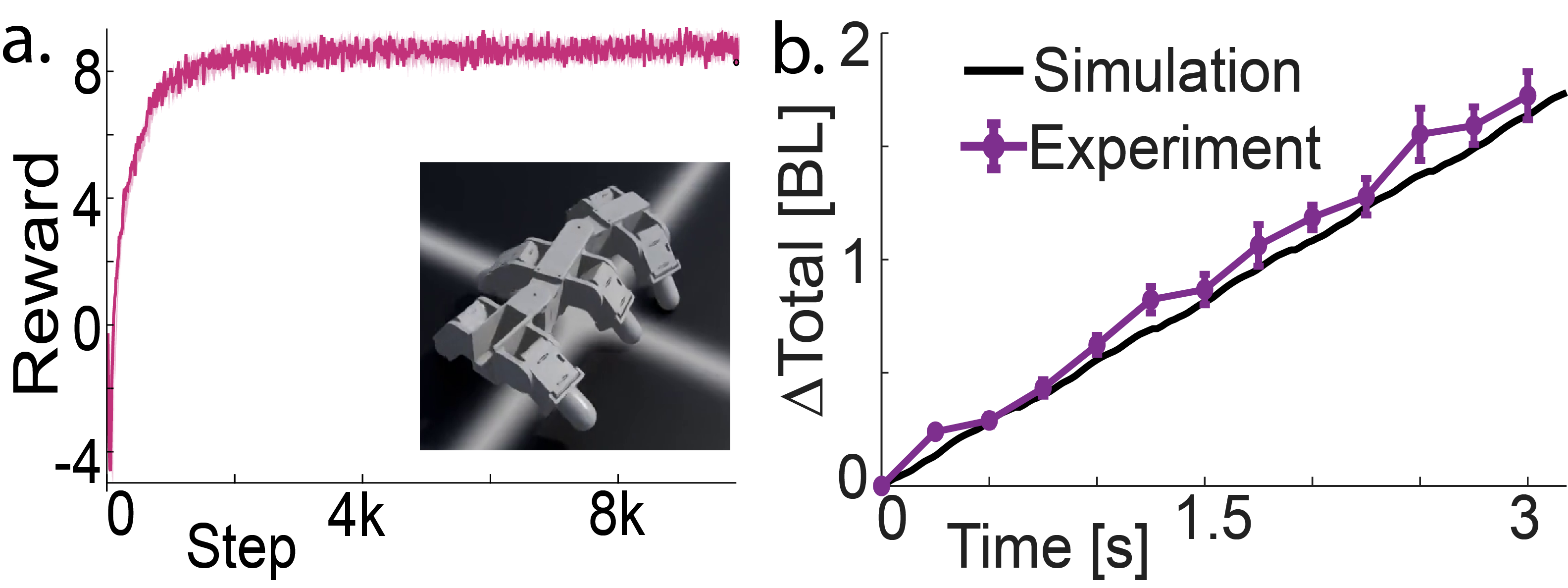}
\caption{\textbf{Open loop RL gait training and sim-to-real transfer.} (a) Plot of mean reward during open-loop reinforcement learning training. Training was conducted over $9,800$ steps with $4096$ parallel environments per step. (\textit{Illustrated}) Hexapod model in Isaac Sim performing learned gait. (b) Comparison of simulation and experiment displacement of learned RL gait. The plot compares 5 cycles of the gait with fitted time scales. Real-world deployment of RL gait consisted of gait cycles $\approx 20\%$ longer than simulation. The results show low sim-to-real gap of the learned gait.
}\vspace{-1em}
\label{fig:RL}
\end{figure}

\vspace{-0.5em}
\section{Discussion and Conclusion}
\vspace{-0.2em}
In this paper, we develop a framework to capture and apply the additional symmetries available in many-legged systems to provide agile forward locomotion. To do this, we design a geometric mechanics model for multi-legged locomotion and make carefully chosen simplifications to translate the model into a graph optimization problem. Furthermore, we provide an approach to use the duality of Ising and Potts spin models to solve this optimization problem in polynomial time and explore the complex gait space of a hexapod robot. Results of our approach provide significant improvements in forward translation for hexapod robots, verified through simulation and robophysical experiments. Our work provides a general approach for studying the locomotion of multi-legged robots and motivates future integration of machine learning to improve the optimization.

We note that although our physics-informed approach is primarily aimed at deriving principled control strategies for multi-legged robots, the same framework can also be used to guide robot design. For example, it enables the systematic design of novel morphologies, such as a hexapod robot in which certain legs are replaced by non-actuated body parts. Beyond solving the associated graph optimization problem, our method provides a framework to reason about which vertices in the graph can be added or removed without degrading locomotion performance. In robotic terms, this framework allows us to quantify how adding, removing, or relocating legs affects overall locomotion, thereby enabling co-design of morphology and control for highly redundant robotic systems.

We acknowledge several limitations of the present framework. Our approach assumes that an environment model has been established a priori; specifically, we consider (1) flat, noise-free terrain and (2) full knowledge of the environment. In many real-world deployment scenarios, such information may be incomplete or unavailable. To address these limitations, future work will extend the framework in two directions. First, we will evaluate graph edges using experimentally measured performance rather than numerically computed models. Second, we will explicitly introduce uncertainty in contact switching and map this uncertainty to finite-temperature effects in the corresponding Ising model, enabling principled reasoning under environmental and actuation uncertainty.

Finally, our model provides physical insight into robophysical systems (simple robots designed to test fundamental locomotion principles). An open question remains as to how these physics-based principles can be scaled to more complex, full-scale robotic platforms. We posit that physics-informed gaits derived from our framework can serve as effective priors or initialization seeds for learning-based controllers, such as reinforcement learning. This strategy would enable robots to adapt to environmental variability while retaining a physically grounded structure. This hybrid approach offers a promising pathway toward an advanced hexapod and, more broadly, highly redundant robots that can outperform their few-legged counterparts.

\bibliographystyle{plainnat}
\bibliography{references}

\end{document}